\documentclass{article}

\usepackage{tikz}
\usepackage{wrapfig}
\usepackage{graphicx}
\usepackage{amsmath}
\usepackage{amssymb}
\usepackage{amsthm}
\usepackage{booktabs}
\usepackage{tikz}
\usepackage{multicol}
\usepackage{todonotes}
\usepackage{enumitem}
\setlist[itemize]{align=parleft,left=0pt..1em}
\usepackage[square,sort,comma,numbers]{natbib}

\newtheorem{example}{Example}
\newtheorem{definition}{Definition}
\newtheorem{lemma}{Lemma}
\newtheorem*{lemma*}{Lemma}

\newcommand{\etal}{\emph{et~al.}}

\usepackage[pagebackref,breaklinks,colorlinks]{hyperref}

\usepackage[final]{neurips_2023}

\usepackage[utf8]{inputenc} %
\usepackage[T1]{fontenc}    %
\usepackage{hyperref}       %
\usepackage{url}            %
\usepackage{booktabs}       %
\usepackage{amsfonts}       %
\usepackage{nicefrac}       %
\usepackage{microtype}      %
\usepackage{xcolor}         %

\title{Implicit Convolutional Kernels for Steerable CNNs}

\author{%
  Maksim Zhdanov\thanks{Work done while at Helmholtz-Zentrum Dresden-Rossendorf.} \\
  AMLab, University of Amsterdam \\
  \texttt{m.zhdanov@uva.nl} \\
  \And
  Nico Hoffmann \\
  Helmholtz-Zentrum \\ Dresden-Rossendorf \\
  \And
  Gabriele Cesa \\
  Qualcomm AI Research\thanks{Qualcomm AI Research is an initiative of Qualcomm Technologies, Inc.} \\
  AMLab, University of Amsterdam
}

\begin{document}

\maketitle

\begin{abstract}
Steerable convolutional neural networks (CNNs) provide a general framework for building neural networks equivariant to translations and transformations of an origin-preserving group $G$, such as reflections and rotations. They rely on standard convolutions with $G$-steerable kernels obtained by analytically solving the group-specific equivariance constraint imposed onto the kernel space. As the solution is tailored to a particular group $G$, implementing a kernel basis does not generalize to other symmetry transformations, complicating the development of general group equivariant models. We propose using implicit neural representation via multi-layer perceptrons (MLPs) to parameterize $G$-steerable kernels. The resulting framework offers a simple and flexible way to implement Steerable CNNs and generalizes to any group $G$ for which a $G$-equivariant MLP can be built. We prove the effectiveness of our method on multiple tasks, including N-body simulations, point cloud classification and molecular property prediction. %
\end{abstract}

\section{Introduction}
\label{sec:intro}

Equivariant deep learning is a powerful tool for high-dimensional problems with known data domain symmetry. By incorporating this knowledge as inductive biases into neural networks, the hypothesis class of functions can be significantly restricted, leading to improved data efficiency and generalization performance \cite{Bronstein2021GeometricDL}. Convolutional neural networks \cite{LeCun1998ConvolutionalNF}  (CNNs) are a prominent example as they are equivariant to translations. Group-equivariant CNNs \cite{Cohen2016GroupEC} (G-CNNs) generalize CNNs to exploit a larger number of symmetries via group convolutions, making them equivariant to the desired symmetry group, such as the Euclidean group $E(n)$ that encompasses translations, rotations, and reflections in $n$-dimensional Euclidean space. In physics and chemistry, many important problems, such as molecular modelling or point clouds, rely on the Euclidean group. Objects defined in physical space have properties that are invariant or equivariant to Euclidean transformations, and respecting this underlying symmetry is often desired for the model to perform as expected.

Neural networks can be parameterized in various ways to incorporate equivariance to the Euclidean group. One option is to use a message-passing neural network as a backbone and compute/update messages equivariantly via convolutions. This approach generalizes well to point clouds and graphs and offers the high expressivity of graph neural networks. Equivariant convolutional operators can be further categorized as regular \cite{Cohen2016GroupEC,Kondor2018OnTG,Bekkers2020BSplineCO, Finzi2020GeneralizingCN} or steerable group convolutions \cite{cohen_steerable_2016, Weiler20183DSC, generaltheory}. The latter recently proved to be especially suitable for incorporating physical and geometric quantities into a model \cite{Brandstetter2022GeometricAP}. The key idea behind Steerable CNNs is using standard convolution - which guarantees translation equivariance - with $G$-steerable kernels that ensure commutativity with the transformations of another group $G$, such as rotations.
The commutation requirement imposes a constraint onto the kernel space that must be solved analytically for each group $G$. This, in turn, does not allow generalizing a convolution operator tailored to a specific group to other symmetry transformations. In the case of the Euclidean group, Cesa \etal \cite{Cesa2022APT} proposed a generally applicable way of parameterizing steerable convolutions for sub-groups of $E(n)$. The method relies on adapting a pre-defined kernel basis explicitly developed for the group $E(n)$ to an arbitrary sub-group by using \emph{group restriction}. 

\begin{wrapfigure}{r}{0.5\textwidth}
\vspace*{-\intextsep}
\centering
\begin{tikzpicture}
\node[inner sep=0pt] at (0,0) {\includegraphics[width=0.49\textwidth,trim={0 2cm 0 1cm},clip]{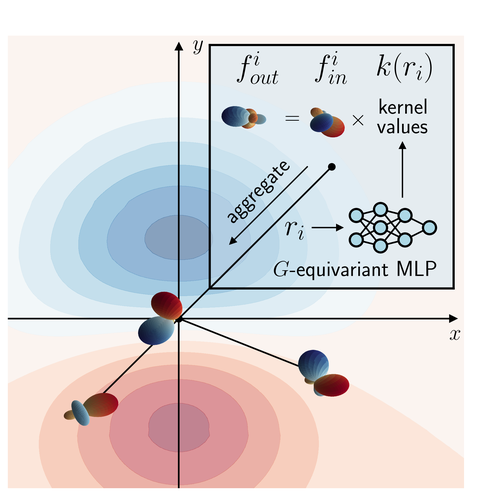}};
\node[anchor=north west, align=left, font=\small, fill=white, fill opacity=0.5, text opacity=1, rounded corners=5pt, text width=1.95cm] at (-3.25,2.7) {
    \textbullet~Flexible \& simple strategy $\forall \; G \leq O(n)$ \\ \vspace{2pt}
    \textbullet~Allows for edge attributes \\ \vspace{2pt}
    \textbullet~Task specific steer. kernels\\ \vspace{2pt}
};
\node[anchor=north west, align=left, font=\normalsize, fill=white, fill opacity=0.5, text opacity=1, rounded corners=5pt] at (-0.9,-2.15) {
    kernel values
};
\end{tikzpicture}
\caption{Illustration of the proposed approach: computing the response of an implicit kernel $k$ (background) of a steerable point convolution for the node $i$ (upper right corner) of a graph with steerable features (visualized as spherical functions).}
\label{fig:intro}
\vspace{-7pt}
\end{wrapfigure} 
However, because only a finite basis can be chosen, a basis tailored for $E(n)$ can be sub-optimal in terms of expressiveness for its sub-groups; see Section~\ref{sec:comparison_analytical}
more details. %
Hence, we propose an alternative way of building steerable convolutions based on implicit neural kernels, i.e. convolutional kernels implemented as continuous functions parameterized by MLPs \cite{Romero2022CKConvCK, romero2021flexconv}. We demonstrate how $G$-steerable convolutions with implicit neural kernels can be implemented from scratch for any sub-group $G$ of the orthogonal group $O(n)$.
The method allows us to ultimately minimize requirements to implement equivariance to new groups; see Section~\ref{sec:g.mlp}.
The flexibility of neural functions also permits the injection of geometric and physical quantities in point convolutions, increasing their expressiveness \cite{Brandstetter2022GeometricAP}; see Section~\ref{sec:k.expand}. We validate our framework on synthetic N-body simulation problem, point-cloud data (ModelNet-40 \cite{Wu20153DSA}) and molecular data (QM9 \cite{Wu2017MoleculeNetAB}) and demonstrate key benefits of our approach such as flexibility and generalizability. Besides, we demonstrate that implicit kernels allow Steerable CNNs to achieve performance competitive with state-of-the-art models and surpass them with the correct choice of the task-specific symmetry group.

\section{Background: Steerable Convolutions}
\label{sec:background}

In this work, we propose a general solution to easily build Steerable CNNs equivariant to translations \emph{and} any compact group $G$\footnote{We provide the definition for a group in Appendix \ref{sec:app.def}.}.
In Section~\ref{sec:g.repr}, we provide some necessary prerequisites from group theory and group representation theory \cite{serre1977linear}.
Then, we review the framework of Steerable CNNs and discuss the constraint it induces on the convolutional kernels in Section~\ref{sec:s_cnn}.

\subsection{Groups, Representations and Equivariance}
\label{sec:g.repr}

\begin{definition}[Group action]
An action of a group $G$ on a set $\mathcal{X}$ is a mapping $(g,x) \mapsto g.x$ associating a group element $g \in G$ and a point $x \in \mathcal{X}$ with some other point on $\mathcal{X}$ such that the following holds:
\begin{equation*}
    g.(h.x) = (gh).x \qquad \forall g,h \in G, x \in \mathcal{X}
\end{equation*}
\end{definition}

\begin{definition}[Group representation]
A linear representation $\rho$ of a group $G$ is a map $\rho : G \rightarrow \mathbb{R}^{d \times d}$ that assigns an invertible matrix $\rho(g) \; \forall g \in G$ and satisfies the following condition
\begin{equation*}
    \rho(gh) = \rho(g)\rho(h) \; \forall g,h \in G.
\end{equation*}
\end{definition}
\noindent A group representation $\rho(g): V \rightarrow V$ furthermore can be seen as a linear action of $G$ on a vector space $V$.
Additionally, if two (or more) vectors $v_1\in \mathbb{R}^{d_1}$ and $v_2 \in \mathbb{R}^{d_2}$ belong to vector spaces transforming under representations $\rho_1$ and $\rho_2$, their concatenation $v_1 \oplus v_2 \in \mathbb{R}^{d_1 + d_2}$ transforms under the \emph{direct sum} representation $\rho_1 \oplus \rho_2$.
$(\rho_1 \oplus \rho_2)(g)$ is a $d_1 + d_2$ dimensional block-diagonal matrix containing $\rho_1(g)$ and $\rho_2(g)$ in its diagonal blocks.

There are three types of group representations that are important for the definition of our method:
\begin{itemize}
    \item \textbf{Trivial representation:} all elements of $G$ act on $V$ as the identity mapping of $V$, i.e. $\rho(g) = I$.
    \item \textbf{Standard representation:} the group $O(n)$ of all orthogonal $n\times n$ matrices has a natural action on $V=\mathbb{R}^n$; similarly, if $G$ is a subgroup of $O(n)$,  elements of $G$ can act on $V=\mathbb{R}^n$ via the inclusion mapping, i.e. $\rho_{st}(g) = g \in \mathbb{R}^{n \times n}$.
    \item \textbf{Irreducible representations} is a collection of generally known representations that can be used as a building block for larger representations via \emph{direct sum}. As argued in \cite{Weiler2019}, Steerable CNNs can be parameterized without loss of generality solely in terms of irreducible representations.
\end{itemize}

\vspace{5pt}
\begin{definition}[Equivariance]
Let us have two spaces $\mathcal{X}, \mathcal{Y}$ endowed with a symmetry group $G$, i.e. with an action defined on them. A function $\phi: \mathcal{X} \rightarrow \mathcal{Y}$ is called \emph{$G$-equivariant}, if it commutes with the action of $G$ on the two spaces, i.e. $\phi(g.x) = g.\phi(x)$ for all $g \in G$, $x \in X$.
\end{definition}
As we have discussed, the layers of conventional CNNs are translation equivariant by design;
however, they do not commute with transformations of other groups, such as rotations and reflections.

\subsection{Steerable CNNs}
\label{sec:s_cnn}
Steerable CNNs provide a more general framework that allows building convolutions that are equivariant to a group of \emph{isometries} of $\mathbb{R}^n$, i.e. $(\mathbb{R}^n, +) \rtimes G \leq E(n)$. Those groups are decomposable as a semi-direct\footnote{See Appendix \ref{sec:app.def} for the definition of a semi-direct product.} product of the translations group $(\mathbb{R}^n, +)$ and an origin-preserving compact\footnote{To remain in the scope of the manuscript, we abstain from the mathematical definition of compact groups, which requires introducing topological groups. One can find more information about compact groups in \cite{Kondor2018OnTG}.} group $G \leq O(n)$, where $O(n)$ is the group of $n$-dimensional rotations and reflections. As translation equivariance is guaranteed \cite{Weiler20183DSC} by the convolution operator itself, one only has to ensure equivariance to $G$.
See \cite{weiler2023EquivariantAndCoordinateIndependentCNNs} for a more in-depth description of Steerable CNNs.

The feature spaces of Steerable CNNs are described as collections of \emph{feature fields}. A feature field of type $\rho$ is a feature map $f: \mathbb{R}^n \to \mathbb{R}^d$ endowed with a \emph{group representation} $\rho: G \to \mathbb{R}^{d \times d}$ that defines how an element $g \in G$ transforms the feature:
\begin{equation}
    [g.f](x) := \rho(g) f(g^{-1}.x)
\end{equation}
Furthermore, each convolutional layer is a map between feature fields. For the map to be equivariant, it must preserve the transformation laws of its input and output feature fields. In practice, it means that the following constraint onto the space of convolution kernels must be applied \cite{Weiler20183DSC}:
\begin{equation}
\label{eq:eq.constraint}
    k(g.x) = \rho_{out}(g) k(x) \rho_{in}(g)^{-1} \qquad \forall g \in G, x \in \mathbb{R}^n
\end{equation}
where $k: \mathbb{R}^n \rightarrow \mathbb{R}^{d_{out} \times d_{in}}$, and $\rho_{in}: G\to  \mathbb{R}^{d_{in} \times d_{in}}$, $\rho_{out} : G \to \mathbb{R}^{d_{out} \times d_{out}}$ are respective representations of input and output feature fields. 

To parameterize the kernel $k$, the constraint in equation \ref{eq:eq.constraint} needs to be solved \emph{analytically} for each specific group $G$ of interest. This renders a general solution challenging to obtain and limits the applicability of steerable convolutions.

\section{Implicit neural kernels}
Instead of deriving a steerable kernel basis for each particular group $G$, we propose parameterizing the kernel $k: \mathbb{R}^n \to \mathbb{R}^{d_{out} \times d_{in}}$ with an MLP satisfying the constraint in Eq.~\ref{eq:eq.constraint}. The approach only requires the $G$-equivariance of the MLP and suggests a flexible framework of implicit steerable convolutions that generalizes to arbitrary groups $G \leq O(n)$.
We argue about the minimal requirements of this approach in Section~\ref{sec:g.mlp}.

We first define the kernel as an equivariant map between vector spaces that we model with an MLP (see Section \ref{sec:vec.k}). Then, we demonstrate that $G$-equivariance of an MLP is a sufficient condition for building the implicit representation of steerable kernels for \emph{compact} groups. 
We indicate that the flexibility of neural representation allows expanding the input of a steerable kernel in Section \ref{sec:k.expand}. Next, we describe how a $G$-equivariant MLP can be implemented in section \ref{sec:g.mlp}. Later, we describe how one can implement $G$-steerable point convolution in the form of equivariant message passing \cite{Satorras2021EnEG, Brandstetter2022GeometricAP} in Section \ref{sec:point.conv} and its generalization to dense convolution in Section \ref{sec:dense.conv}. Finally, we compare our method with the solution strategy proposed in \cite{Cesa2022APT} in Section \ref{sec:comparison_analytical}.

\subsection{Kernel vectorization and equivariance}
\label{sec:vec.k}
Our goal is to implement the kernel $k : \mathbb{R}^{n} \rightarrow \mathbb{R}^{d_{out} \times d_{in}}$ of a $G$-steerable convolution that maps between spaces of feature fields with representations $\rho_{in}$ and $\rho_{out}$.
The kernel itself is a function whose input in $\mathbb{R}^n$ transforms under the \emph{standard representation} $\rho_{st}$
(as defined in Section~\ref{sec:g.repr})
and which we will model with an MLP. 
Since MLPs typically output vectors, it is convenient to \emph{vectorize} the $d_{out} \times d_{in}$ output of the kernel. We denote the column-wise vectorization of a matrix $M \in \mathbb{R}^{d_1\times d_2}$ as $vec(M) \in \mathbb{R}^{d_1d_2}$. Henceforth, we will consider kernel's vector form $vec(k(\cdot)):  \mathbb{R}^{n} \rightarrow \mathbb{R}^{d_{out} d_{in}}$.

Let $\otimes$ denote the \emph{Kronecker product} between two matrices.
Then, $\rho_\otimes(g) := \rho_{in}(g) \otimes \rho_{out}(g)$ is also a representation\footnote{This representation is formally known as the \emph{tensor product} of the two representations.} of $G$.
We suggest an implicit representation of the vectorized kernel $k$ using an $G$-equivariant MLP $\phi: \mathbb{R}^{n} \rightarrow \mathbb{R}^{d_{out} d_{in}}$ based on the following lemma (see \ref{sec:app.proof} for the proof):
\begin{lemma}
\label{eq:lemma}
If a kernel $k$ is parameterized by a $G$-equivariant MLP $\phi$ with input representation $\rho_{st}$ and output representation $\rho_\otimes := \rho_{in} \otimes \rho_{out}$ , i.e. $vec(k)(x) := \phi(x)$, then the kernel satisfies the equivariance constraint in Equation~\ref{eq:eq.constraint} for a compact group $G$.
\end{lemma}

In other words, $G$-equivariance of MLP is a sufficient condition for $G$-equivariance of the convolutional layer whose kernel it parameterizes. Using implicit kernels also has a very favourable property - it allows arbitrary steerable features as its input, which we discuss in the following section.

\subsection{Expanding the input}
\label{sec:k.expand}
Note that in the case of standard steerable convolutions, the input $x \in \mathbb{R}^n$ of a kernel is usually only the difference between the spatial position of two points. However, there is no requirement that would disallow the expansion of the input space except for practical reasons. 
Hence, here we augment steerable kernels with an additional feature vector $z \in \mathbb{R}^{d_z}$.
This formulation allows us to incorporate relevant information in convolutional layers, such as physical and geometric features. For example, when performing convolutions on molecular graphs, $z$ can encode the input and output atoms' types and yield different responses for different atoms.
If introducing additional arguments into a kernel, the steerability constraint in equation~\ref{eq:eq.constraint} should be adapted to account for transformations of $z$:
\begin{equation}
\label{eq:eq.constraint_z}
    k(g.x,\rho_z(g)z) = \rho_{out}(g) k(x,z) \rho_{in}(g)^{-1}
\end{equation}
which must hold for all $ g \in G, x \in \mathbb{R}^n, z \in \mathbb{R}^{d_z}$, where $\rho_z: G \to \mathbb{R}^{d_z \times d_z}$ is the representation of $G$ acting on the additional features.

Again, analytically solving the constraint \ref{eq:eq.constraint_z} to find a kernel basis for arbitrary $\rho_z$ is generally unfeasible.
Note also that the solution strategy proposed in \cite{Cesa2022APT} requires a basis for functions over $\mathbb{R}^{n + d_z}$, whose size tends to grow exponentially with $d_z$ and, therefore, is not suitable. 
Alternatively, we can now use the flexibility of neural representation and introduce additional features into a kernel at no cost. 

\subsection{Implementing a $G$-equivariant MLP}
\label{sec:g.mlp}
We are now interested in how to build an MLP that is equivariant to the transformations of the group $G$, i.e. a sequence of equivariant linear layers alternated with equivariant non-linearities \cite{Shawe-Taylor_1989, universalgroupmlp, Finzi2021APM}. It is important to say that our approach does not rely on a specific implementation of $G$-MLPs, and any algorithm of preference might be used (e.g. \cite{Finzi2021APM} or enforcing via an additional loss term). 

The approach we employed in our experiments is described below. Since the irreducible representations of $G$ are typically known\footnote{
    This is a reasonable assumption since irreducible representations are often used to construct other representations used in Steerable CNNs. If not, there exist numerical methods to discover them from the group algebra.
}, one can always rely on the following properties: \emph{1)} any representation $\rho$ of a compact group $G$ can be decomposed as a direct sum of irreducible representations $\rho(g) = Q^T \left(\bigoplus_{i \in I} \psi_i(g) \right) Q$ with a change of basis $Q$\footnote{\cite{Cesa2022APT} describes and implements a numerical method for this.} and \emph{2)} \emph{Schur's Lemma}, which states that there exist equivariant linear maps only between the irreducible representation of the same kind\footnote{Typically, these maps only include scalar multiples of the identity matrix.}.
Hence, one can apply the right change of basis to the input and output of a linear layer and then learn only maps between input and output channels associated with the same irreducible representations.
In the context of implicit kernels, the tensor product representation $\rho_{in} \otimes \rho_{out}$ in the last layer is decomposed by a matrix containing the Clebsch-Gordan coefficients, which often appears in the analytical solutions of the kernel constraint in the related works. %
Note that the non-linearity $\sigma$ used in the Steerable CNNs are $G$-equivariant and, therefore, can be used for the MLP as well.

\subsection{$G$-steerable point convolutions}
\label{sec:point.conv}
Point convolutions operate on point clouds - sets of $N$ points endowed with spatial information %
$X = \{x_i\}_{i=0}^{N-1} \in \mathbb{R}^{n \times N}$. A point cloud thus provides a natural discretization of the data domain, which renders a convolution operator as follows:
\begin{equation}
\label{eq:sg.cnn}
    f_{out}(x_i) = (k \ast f_{in})(x_i) = \sum\limits_{0 \leq j \leq N-1} k(x_i-x_j)f_{in}(x_j)
\end{equation}
To reduce the computational cost for large objects, one can induce connectivity onto $x \in X$ and represent a point cloud as a graph $\mathcal{G} = (\mathcal{V}, \mathcal{E})$ with nodes $v_i \in \mathcal{V}$ and edges $e_{ij} \in \mathcal{E} \subseteq \mathcal{V} \times \mathcal{V}$, where each node $v_i$ has a spatial location $x_i$, node features $z_i$ and a corresponding feature map $f_{in}(x_i)$ (see Figure \ref{fig:intro}). Additionally, each edge $e_{ij}$ can have an attribute vector $z_{ij}$ assigned to it (as in Section~\ref{sec:k.expand}). This allows for a learnable message-passing point convolution whose computational cost scales linearly with the number of edges:
\begin{equation}
\label{eq:sg.cnn_z}
    f_{out}(x_i) = (k \ast f_{in})(x_i) = \sum\limits_{j \in \mathcal{N}(i)} k(x_i-x_j, z_i, z_j, z_{ij})f_{in}(x_j)
\end{equation}
where $\mathcal{N}(i) = \{j : (v_i, v_j) \in \mathcal{E}\}$ and the kernel $k(\cdot)$ is parameterized by a $G$-equivariant MLP.

\subsection{Extension to $G$-steerable CNNs}
\label{sec:dense.conv}
Straightforwardly, the proposed method can be extended to dense convolutions. In such case, the kernel is defined as $k : \mathbb{R}^{n} \rightarrow \mathbb{R}^{c_{out} \times c_{in} \times K^{n}}$ - that is, a continuous function that returns a collection of $K \times K \times...$ kernels given a relative position (and, optionally, arbitrary steerable features). In this case, the vectorized kernel is parameterized in exactly the same way as described above but is estimated at the center of each pixel.

\subsection{Comparison with the analytical solution}
\label{sec:comparison_analytical}

A general basis for any $G$-steerable kernel is described in \cite{Cesa2022APT}. Essentially, it relies on two ingredients:
\emph{i)} a (pre-defined) finite $G$-steerable basis \cite{Freeman1991-STEER} for \emph{scalar} filters and \emph{ii)} a learnable equivariant linear map. Let us look at those in detail. Firstly, a finite $G$-steerable basis is essentially a collection of $B$ orthogonal functions, i.e. $Y: \mathbb{R}^n \to \mathbb{R}^B$, with the following equivariant property: $Y(g.x) = \rho_Y(g) Y(x)$, for some representation $\rho_Y$ of $G$.
The linear map in \emph{ii)}, then, is a general equivariant linear layer, whose input and output transform as $\rho_Y$ and $\rho_{in} \otimes \rho_{out}$.

In practice, it means that one has to provide a pre-defined basis $Y$ for a group of interest $G$. Since the design might not be straightforward, \cite{Cesa2022APT} suggest a way to \emph{reuse} an already derived $O(n)$ steerable basis for any subgroup $G \subset O(n)$. While general, such a solution can be sub-optimal. For example, if $n=3$, an $O(3)$-steerable basis has local support inside a sphere, which is suitable for the group of 3D rotations $SO(3)$ but not ideal for cylindrical symmetries, i.e. when $G$ is the group $SO(2)$ of planar rotations around the Z axis. 

In comparison, the solution proposed in Section \ref{sec:vec.k} replaces the pre-defined basis $Y$ with a learnable $G$-MLP. It allows us to learn $G$-specific kernels without relying on a pre-derived basis for a larger group, which in turn means that we can theoretically obtain $G$-optimal kernel basis via learning (see \ref{sec:app.comparison_analytical} for further details). Furthermore, the kernels defined in Section \ref{sec:vec.k} can now be conditioned on arbitrary steerable features, which makes them more task-specific and expressive. In the context of a general basis, one can interpret the last linear layer of an implicit kernel as the map in \emph{ii)}, and the activations before this layer as a learnable version of the basis $Y$ in \emph{i)}.

\section{Related works}

\paragraph{Group convolutions.} 
Multiple architectures were proposed to achieve equivariance to a certain symmetry group. It has been proven \cite{Kondor2018OnTG} that the convolutional structure is a sufficient condition for building a model that is equivariant to translations and actions of a compact group. One can separate group convolutions into two classes depending on the space on which the convolution operates: \emph{regular} \cite{Cohen2016GroupEC,Kondor2018OnTG,bekkers2018roto,Bekkers2020BSplineCO, Finzi2020GeneralizingCN} and \emph{steerable} \cite{cohen_steerable_2016, Worrall2017-HNET,Weiler20183DSC, Thomas2018TensorFN, generaltheory, Weiler2019, Cesa2022APT} group convolutions. In the first case, the input signal is represented in terms of scalar fields on a group $G$, and the convolution relies on a discretization of the group space. Steerable convolutions are a class of $G$-equivariant convolutions that operate on feature fields over homogeneous spaces and achieve equivariance via constraining the kernel space. They further avoid the discretization of the group space and can reduce the equivariance error in the case of continuous groups.

\paragraph{$G$-steerable kernels.} In the case of rigid body motions $G = SO(3)$, the solution of the equivariance constraint is given by spherical harmonics modulated by an arbitrary continuous radial function, which was analytically obtained by Weiler \etal \cite{Weiler20183DSC}. Lang and Weiler \cite{lang2020wigner} then applied the Wigner-Eckart theorem to parametrize $G$-steerable kernel spaces over orbits of a compact $G$. The approach was later generalized by Cesa \etal \cite{Cesa2022APT}, who proposed a solution for any compact sub-group of $O(3)$ based on group restriction. Using this approach, one can obtain a kernel basis for a group $G \leq  H$ if the basis for $H$ is known. Despite its generalizability, the method still requires a pre-defined basis for the group $H$ that is further adapted to $G$. The resulting solution is not guaranteed to be optimal for $G$; see Section~\ref{sec:comparison_analytical}. We note that the practical value of steerable kernels is also high as they can be used for convolution over arbitrary manifolds in the framework of Gauge CNNs \cite{cohen2019gauge,de2020gauge,weiler2021coordinateIndependentGauge}.

\paragraph{Implicit kernels.}
Using the implicit representation of convolution kernels for regular CNNs is not novel. It was used, for example, to model long-range dependencies in sequential data \cite{Romero2022CKConvCK} or for signal representation \cite{Sitzmann2020ImplicitNR}. Romero \etal \cite{Romero2022CKConvCK} demonstrated that such parametrization allows building shallower networks, thus requiring less computational resources to capture global information about the system. Continuous kernels were recently used to build an architecture \cite{Knigge2023ModellingLR} for processing data of arbitrary resolution, dimensionality and length, yet equivariant solutions are scarce \cite{Ouderaa2022RelaxingEC}. Finzi \etal \cite{Finzi2020GeneralizingCN} proposed parametrizing convolutions on Lie groups as continuous scalar functions in the group space. The method relies on discretizing a continuous group, which might lead to undesirable stochasticity of the model's output. Instead, we use Steerable CNNs that define the kernel as a function on a homogeneous space. While the discretization of this space is still required, in most cases, it is naturally given by the data itself, e.g. for point clouds; hence, no sampling error arises. It is also important to mention the key difference between implicit kernels and applying $G$-MLPs \cite{Finzi2021APM} directly - the latter is incapable of processing image/volumetric data as convolutions do. Henceforth, we focus on CNNs with consideration for potential extensions to various data modalities.

\paragraph{Equivariant point convolutions.}
A particular momentum has been gained by point convolutions in the form of equivariant message-passing \cite{Schtt2017SchNetAC, Thomas2018TensorFN, Satorras2021EnEG, Brandstetter2022GeometricAP} specifically for problems where symmetry provides a strong inductive bias such as molecular modelling \cite{Anderson2019CormorantCM} or physical simulations \cite{Fuchs2020SE3Transformers3R}. Thomas \etal \cite{Thomas2018TensorFN} pioneered $SE(3)$-equivariant steerable convolutions whose kernels are based on spherical harmonics modulated by a radial function. The approach was further generalized by Batzner \etal \cite{Batzner2022E3equivariantGN}, who uses an MLP conditioned on the relative location to parameterize the radial function, although the basis of spherical harmonics is preserved. 
Brandstetter \etal \cite{Brandstetter2022GeometricAP} demonstrated that introducing geometric and physical information into an equivariant message-passing model improves the expressivity on various tasks, which we also observe in this work. %
Note that, for $G=SO(3)$ or $O(3)$ and without additional edge features $z$, our MLP can only learn a function of the radius and, therefore, is equivalent to the models proposed in \cite{Batzner2022E3equivariantGN}.

\section{Experiments}
\label{sec:experiments}

In this section, we implement Steerable CNNs with implicit kernels and apply them to various tasks\footnote{All datasets were downloaded and evaluated by Maksim Zhdanov (University of Amsterdam).}. First, we indicate the importance of correctly choosing the symmetry group on a synthetic N-body simulation problem where an external axial force breaks the rotational symmetry (see Section~\ref{sec:nbbody}). Then, we prove the generalizability of the proposed approach as well as the gain in performance compared to the method proposed in \cite{Cesa2022APT} on ModelNet-40 (see Section~\ref{sec:mn}). Afterwards, we show that one can introduce additional physical information into a kernel and significantly improve the performance of steerable convolutions on molecular data (see Section~\ref{sec:qm}). Code and data to reproduce all experiments are available on \href{https://github.com/maxxxzdn/implicit-steerable-kernels}{GitHub}.

\subsection{Implementation}
\label{sec:implement}
\paragraph{Implicit kernels.} To parameterize implicit kernels, we employ linear layers (see Section~\ref{sec:g.mlp}) followed by quotient ELU non-linearities \cite{Cesa2022APT}. The last layer generates a steerable vector, which we reshape to yield a final convolutional kernel. The $G$-MLP takes as input a steerable vector obtained via direct sum (concatenation) of batch-normalized harmonic polynomial representation of the relative location $x$, edge features $z_{ij}$ and input node features $z_{i}$. See details on pre-processing and task-specific inputs in Appendix \ref{sec:app.prep}.

\paragraph{Steerable convolutions.} 

We employ steerable point convolutional layers as described in Equation \ref{eq:sg.cnn_z}. For each task, we tune hyperparameters of all models on validation data: the number of layers, the number of channels in each layer, the depth and the width of implicit kernels, and the number of training epochs. For all the experiments, except the one reported in Table~\ref{tab:exp4}, we yield a number of parameters similar to the baselines with a deviation of less than 10\%. For ModelNet-40 experiments, our architecture is partially inspired by the model introduced in \cite{Poulenard2021AFA}, which uses gradual downsampling of a point cloud\footnote{Poulenard \etal \cite{Poulenard2021AFA} use kd-tree pooling to compute coarser point clouds, while we only use random sampling of points. Note also that the spherical quotient ELU we used is similar in spirit to the proposed functional non-linearity described there, yet it does not employ a deep MLP.}. For N-body and QM9 experiments, we add residual connections to each steerable layer to learn higher-frequency features \cite{Giovanni2022GraphNN}. 
For QM9 and ModelNet-40 experiments, the last steerable layer returns a vector with invariant features for each node, to which we apply global pooling to obtain a global representation of an object, which is further passed to a classification MLP. In the case of N-body experiments, we output the standard representation corresponding to the particles' coordinates. Details about optimization and model implementation for each task can be found in Appendix \ref{sec:app.exps}.

\setlength{\belowcaptionskip}{-10pt}
\begin{figure}[t]
    \centering
    \begin{minipage}{0.33\textwidth}
        \centering
        \includegraphics[width=\textwidth]{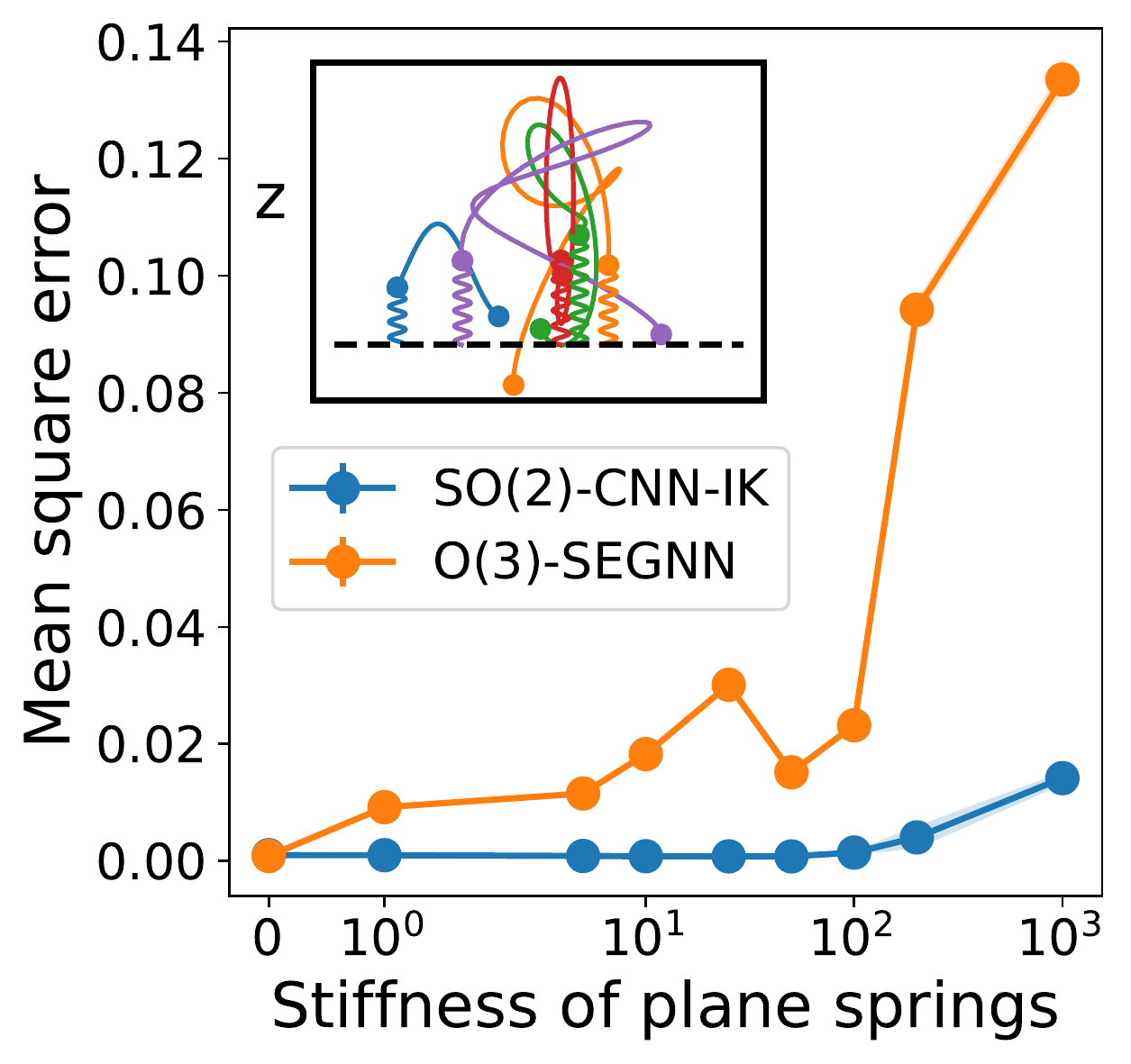} %
        \caption{Final position estimation in the N-body system experiment, where particles are connected to one another and to the XY-plane by springs (see left upper corner). Our model with correct axial symmetry $SO(2)$, significantly outperforms the state-of-the-art model SEGNN \cite{Brandstetter2022GeometricAP}, which is $O(3)$-equivariant, as the relative contribution of the plane string increases.}
        \label{fig:exp1}
    \end{minipage}
    \hfill
    \begin{minipage}{0.63\textwidth}
        \centering
        \includegraphics[width=\textwidth,trim={0 0.5cm 0 0},clip]{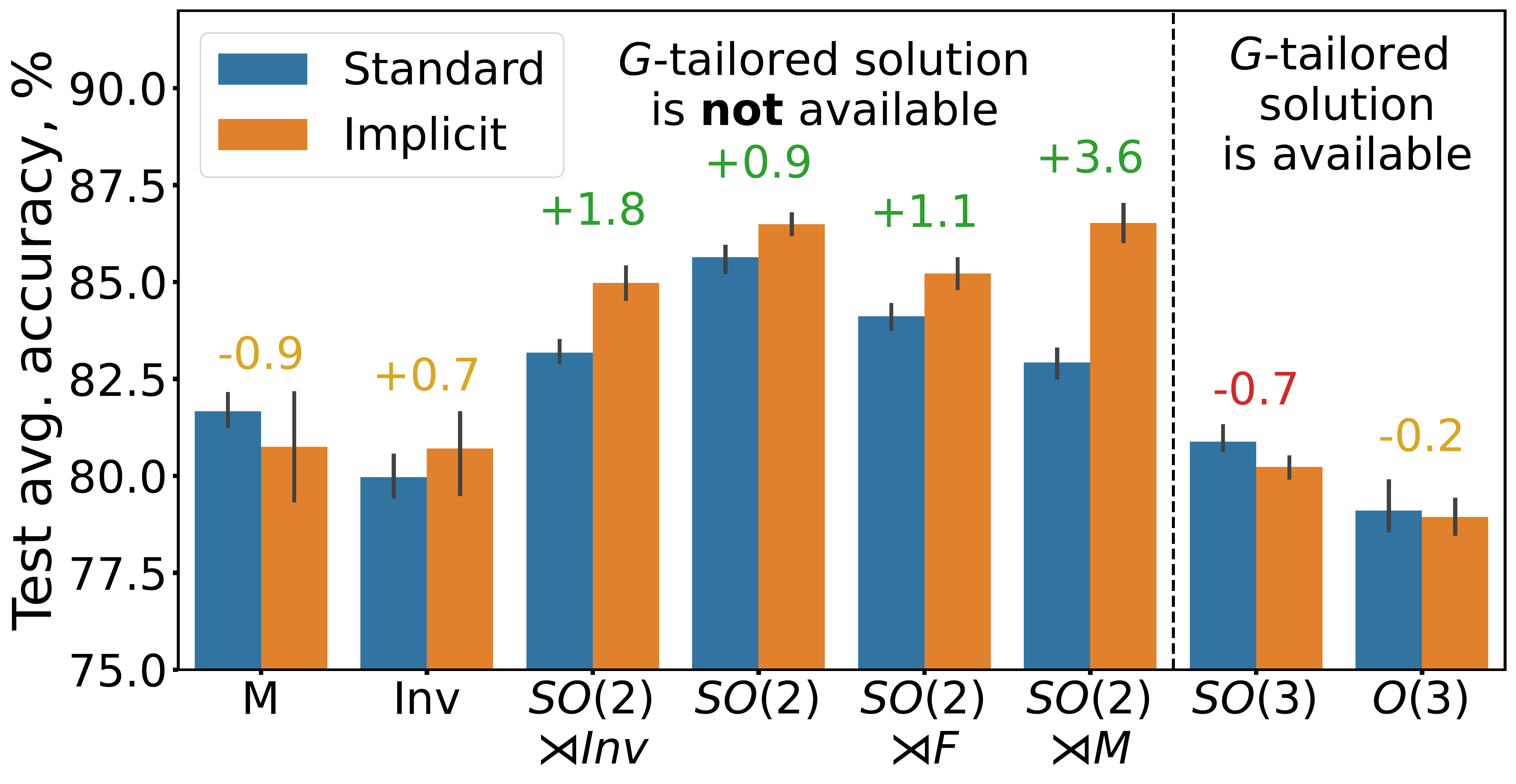} %
        \caption{Performance comparison of Steerable CNNs on the rotated ModelNet-40 dataset for different $G$. Bars show mean average accuracy with error bars indicating standard deviation, both computed from 5 runs. The numbers above bars denote the performance gain of implicit kernels (orange) over \cite{Cesa2022APT} (blue). Statistically significant ($p<0.05$) positive differences are green, negative ones are red, and insignificant ones are yellow. $SO(2)$ contains planar rotations around the Z axis, $M$ and $Inv$ contain mirroring along the $X$ axis and origin respectively, while $F$ contains rotations by $\pi$ around the $X$ axis. $O(2) \cong SO(2) \rtimes M$ achieves the best accuracy as it best represents the real symmetries of the data.}
        \label{fig:exp2}
    \end{minipage}
\end{figure}

\subsection{The relevance of smaller $G<O(n)$: N-body simulation}
\label{sec:nbbody}
\paragraph{Dataset.}
We conduct experiments using the N-body system \cite{kipf2018neural}, where particles are connected through springs, and their interaction is governed by Hooke's law. Similar to previous studies \cite{Satorras2021EnEG, Brandstetter2022GeometricAP}, we modify the original trajectory prediction task to calculate the position of each particle in a 3-dimensional space, given its initial position and velocity. We attach each particle to the XY plane using strings with random equilibrium lengths and pre-defined stiffness, which can slide freely over the plane (see Figure \ref{fig:exp1}, upper left corner). This additional force term breaks the rotational symmetry of the system into only azimuthal symmetry; this system resembles a simplified version of the problem of molecule binding to the surface of a larger molecule. We choose the model's hyperparameters to have a similar parameter budget to the SEGNN model \cite{Brandstetter2022GeometricAP}, which has shown state-of-the-art performance on a similar task (we also compare against a non-equivariant baseline; see Table \ref{tab:app_nbody}). For all models, we use the highest frequency of hidden representations in both $G$-MLP and Steerable CNNs equal to $1$. We use the velocity of a particle and the equilibrium length of the attached XY spring as input; the model's output transforms under the standard representation. We train a separate model for each value of plane strings' stiffness (3000 training points) to measure the impact of symmetry breaks on performance. 

\paragraph{Results.}
As can be seen in Fig. \ref{fig:exp1}, a Steerable CNN with azimuthal symmetry $SO(2)$ significantly outperforms SEGNN, which is equivariant to a larger group $O(3)$. Since we introduced a force term that now breaks the rotational symmetry, SEGNN struggles to learn it. Furthermore, while in the default setting (plane strings' stiffness is 0), models achieve roughly the same performance, the absolute difference grows exponentially once the plane strings are introduced. 

The synthetic problem is meant to show the importance of choosing the correct symmetry when designing a model. While the Euclidean group E(3) is often enough for N body systems in a vacuum, it is important to be careful when an external influence or underlying structure can disrupt global symmetries and negatively impact a model with a larger symmetry group. This can be relevant in scenarios like molecular docking simulations \cite{Corso2022DiffDockDS}, where the system can align with the larger molecule, or material science \cite{Kaba2022EquivariantNF}, where the arrangement of atoms and crystal lattice structures yield a discrete group of symmetries smaller than $O(n)$.

\subsection{Generalizability of implicit kernels: ModelNet-40}
\label{sec:mn}
\paragraph{Dataset.}
The ModelNet-40 \cite{Wu20153DSA} dataset contains 12311 CAD models from the 40 categories of furniture with the orientation of each object aligned. The task is to predict the category of an object based on its point cloud model. 2468 models are reserved for the test partition. From the remaining objects, we take 80\% for training and 20\% for validation. We augment each partition with random rotations around the Z-axis. We induce connectivity on point clouds with a k-nearest neighbour search with $k=10$ at each model layer and use normals as input node features.
\begin{wraptable}{R}{0.4\textwidth}
\vspace*{-\intextsep}
\vspace{-5pt}
\caption{\label{tab:mn}Overall accuracy (OA) on ModelNet-40. Group equivariant methods are denoted by $*$.}
\vspace{15pt}
\centering
\begin{tabular}{lc}
\hline& \\[\dimexpr-\normalbaselineskip+2pt]
Method                                                        & OA, \% \\
\hline& \\[\dimexpr-\normalbaselineskip+2pt]
Spherical-CNN \cite{Esteves2017LearningSE} $*$  & 88.9     \\
SE(3)-ESN \cite{Franzen2021NonlinearitiesIS} $*$ & 89.1     \\
TFN{[}mlp{]} P \cite{Poulenard2021AFA} $*$      & 89.4     \\
PointNet++ \cite{qi2017pointnetplusplus}     & 91.8     \\
SFCNN \cite{Rao2019SphericalFC}              & 92.3     \\
PointMLP \cite{Ma2022RethinkingND}           & 94.5     \\ \hline& \\[\dimexpr-\normalbaselineskip+2pt]
Ours & 89.3 $\pm$ 0.06\\                                       
\hline
\end{tabular}
\vspace{-5pt}
\end{wraptable} 

\paragraph{Results.}
In this section, we demonstrate how one can build a Steerable CNN that is equivariant to an arbitrary subgroup of the Euclidean group $G \rtimes (\mathbb{R}^3,+) \leq E(3)$. We compare the performance of implicit kernels with the standard steerable kernels obtained by group restriction \cite{Cesa2022APT} and keep the number of parameters similar. The results are shown in Figure \ref{fig:exp2}. Implicit kernels achieve significant improvement in accuracy on test data for the majority of groups. The only statistically significant negative difference is presented for $G = SO(3)$, for which a tailored and hence optimal kernel basis is already available. When a custom solution is unknown, implicit kernels often significantly outperform the previously proposed method. Therefore, they pose an efficient toolkit for building a kernel basis for an arbitrary subgroup of $E(3)$. We report the result of the best-performing model in Table~\ref{tab:mn}. Although we do not outperform more task-specific and involved approaches such as PointMLP \cite{Ma2022RethinkingND}, our model is on par with other group equivariant models. We emphasize that our model is a simple stack of convolutional layers and hypothesize that a smarter downsampling strategy, residual connections and hierarchical representations would significantly improve the overall performance of our model, which we leave for further research.

\subsection{Flexibility of implicit kernels: QM9}
\label{sec:qm}

\paragraph{Dataset.}
The QM9 dataset \cite{Wu2017MoleculeNetAB} is a public dataset consisting of about $130$k molecules with up to 29 atoms per molecule. Each molecule is represented by a graph with nodes denoting atoms and edges indicating covalent bonds. Each node is assigned a feature vector consisting of one-hot encoding of the type of the respective atom (H, C, N, O, F) and its spatial information corresponding to a low energy conformation. Additionally, each molecule is described by 19 properties from which we select 12 commonly taken in literature \cite{Brandstetter2022GeometricAP} for regression tasks. Different from common practice \cite{Brandstetter2022GeometricAP}, we perform convolution on molecular graphs, i.e. with connectivity pre-defined by molecular structure instead of inducing it. The design choice is motivated by the presence of edge features, which we include in an implicit kernel to display the flexibility of neural representation. 

\begin{wrapfigure}{R}{0.5\textwidth}
\vspace{5pt}
\centering
\includegraphics[width=0.5\textwidth, trim={0 1cm 0 2cm}]{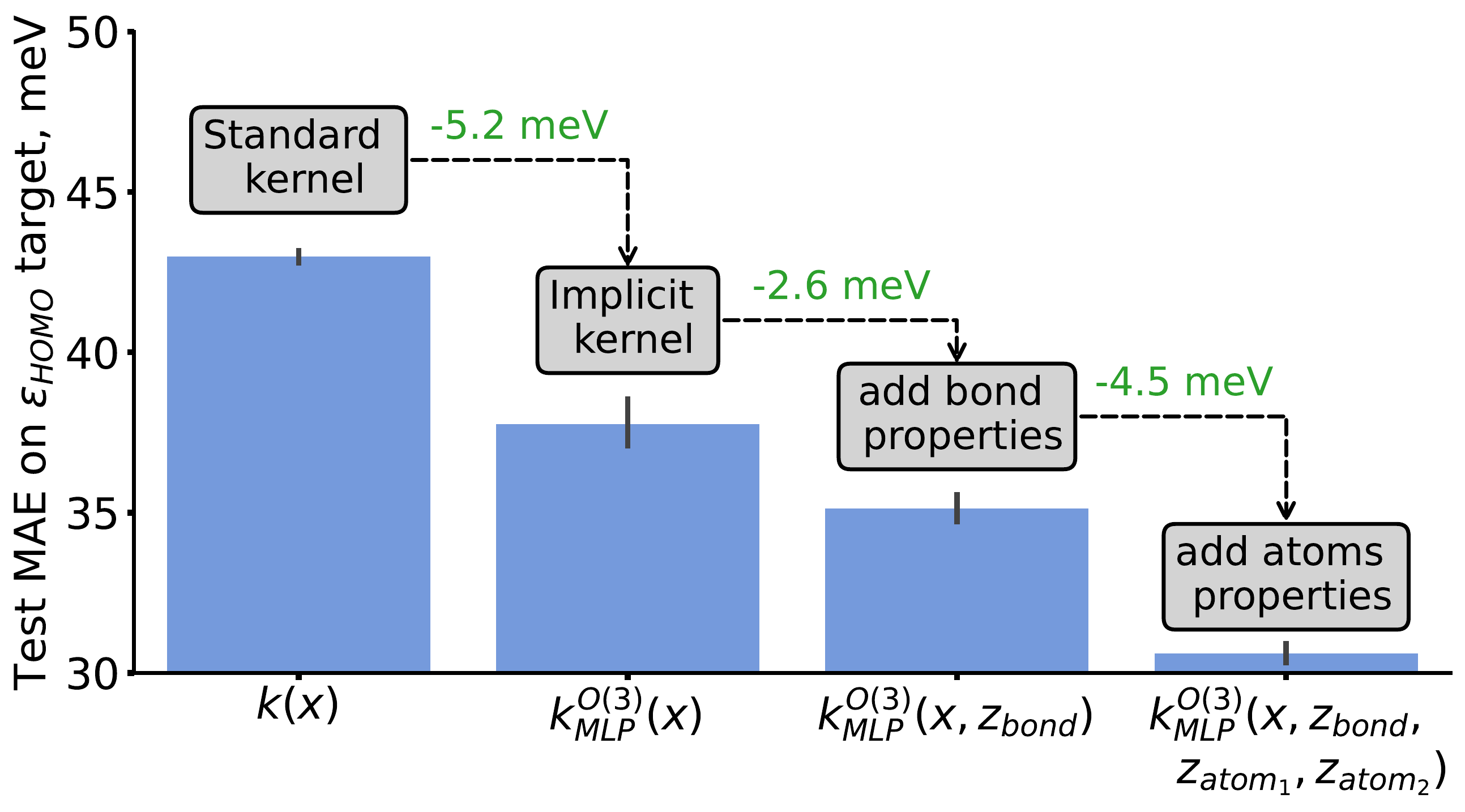}
\caption{Using implicit kernels $k_{MLP}^{G}$ and injecting it with bond and atoms properties significantly improves the performance of Steerable CNNs on the QM9 dataset (Mean Absolute Error on the $\varepsilon_{HOMO}$ regression problem). Bars denote mean average accuracy on the test dataset with error bars corresponding to standard deviation; both computed on 5 runs. The kernel is $O(3)$-equivariant rendering the final architecture $E(3)$-equivariant.}
\label{fig:exp3}
\vspace{+0pt}
\end{wrapfigure} 

\paragraph{Results.}
We first demonstrate how one can use the flexibility of neural representation to introduce additional features of choice into a steerable convolutional layer. As each pair of atoms connected by a covalent bond is assigned a one-hot encoding of the bond type $z_{ij}$, we use it as a condition for $O(3)$-equivariant implicit kernels. Additionally, we follow Musaelian \etal \cite{Musaelian2022LearningLE} and embed one-hot encoding of the center and neighbour atom types $z_{i}$ and $z_{j}$ into the MLP. We include each property one by one and indicate corresponding performance gain in Figure \ref{fig:exp3}. First, we observed a sharp improvement by switching from standard steerable kernels to implicit ones. We attribute it to the higher expressivity of implicit kernels and their ability to learn more complex interactions. Furthermore, injecting edge attributes into the kernel computation reduced MAE even further. Introducing both atom types and an edge type significantly influenced the performance, corresponding to the model learning how to process each specific combination differently, thus adding to its expressivity. It is not surprising since a similar result was obtained by Brandstetter \etal \cite{Brandstetter2022GeometricAP}, who used non-linear message aggregation conditioned on physical information, yielding state-of-the-art performance.

\begin{table}[]
\caption{Mean Absolute Error (MAE) between model predictions and ground truth for the molecular property prediction on the QM9 dataset. Linear steerable convolutions are denoted by $*$. L stands for the number of layers, and W stands for the number of channels in each layer.}
\label{tab:exp4}
\resizebox{\linewidth}{!}{%
\begin{tabular}{lcccccccccccc}
\hline
Task      & $\alpha$     & $\Delta \varepsilon$ & $\varepsilon_{HOMO}$ & $\varepsilon_{LUMO}$ & $\mu$    & $C_\nu$    & G  & H  & $R^2$    & U  & $U_0$ & ZPVE \\ 
Units      & bohr$^3$     & meV & meV & meV & D & cal/mol K    & meV  & meV  &  bohr$^3$  & meV  & meV & meV \\ 
\hline& \\[\dimexpr-\normalbaselineskip+2pt]
NMP\cite{Gilmer2017NeuralMP}      & .092 & 69    & 43   & 38   & .030 & .040 & 19 & 17 & 0.180 & 20 & 20 & 1.50 \\
SchNet\cite{Schtt2017SchNetAC}    & .235 & 63    & 41   & 34   & .033 & .033 & 14 & 14 & 0.073 & 19 & 14 & 1.70 \\
SE(3)-Tr.\cite{Fuchs2020SE3Transformers3R} & .142 & 53    & 35   & 33   & .051 & .054 & -  & -  & -     & -  & -  & -    \\
DimeNet++\cite{Klicpera2020FastAU} & .043  & 32    & 24   & 19   & .029 & .023 & 7  & 6  & 0.331 & 6  & 6  & 1.21 \\
SphereNet\cite{Liu2021SphericalMP} & .046  & 32    & 23   & 18   & .026 & .021 & 8  & 6  & 0.292 & 7  & 6  & 1.12 \\
PaiNN\cite{Schtt2021EquivariantMP} & .045 & 45    & 27   & 20   & .012 & .024 & 7  & 6  & 0.066 & 5  & 5  & 1.28 \\
EGNN\cite{Satorras2021EnEG}    & .071 & 48    & 29   & 25   & .029 & .031 & 12 & 12 & 0.106  & 12 & 12 & 1.55 \\
SEGNN\cite{Brandstetter2022GeometricAP}     & .060 & 42    & 24   & 21   & .023 & .031 & 15 & 16 & 0.660 & 13 & 15 & 1.62 \\
TFN\cite{Thomas2018TensorFN}$*$       & .223 & 58    & 40   & 38   & .064 & .101 & -  & -  & -     & -  & -  & -    \\
Cormorant\cite{Anderson2019CormorantCM}$*$ & .085 & 61    & 34   & 38   & .038 & .026 & 20 & 21 & 0.961 & 21 & 22 & 2.02 \\
L1Net\cite{Miller2020RelevanceOR}$*$     & .088 & 68    & 46   & 35   & .043 & .031 & 14 & 14 & 0.354 & 14 & 13 & 1.56 \\
LieConv \cite{Finzi2020GeneralizingCN}$*$   & .084 & 49    & 30   & 25   & .032 & .038 & 22 & 24 & 0.800  & 19 & 19 & 2.28 \\
\hline& \\[\dimexpr-\normalbaselineskip+2pt]
Ours (W=24, L=15)      & .078 & 45.3    & 24.1  & 22.3   & .033 & .032 & 21.1 & 19.6 & 0.809 & 19.7 & 19.5 & 2.08 \\
Ours (W=16, L=30)      & .077 & 43.5    & 22.8  & 22.7   & .029 & .032 & 19.9 & 21.5 & 0.851 & 22.5 & 22.3 & 1.99 \\ \hline
\end{tabular}
}
\end{table}
\paragraph{Scaling the model up.} Table \ref{tab:exp4} shows the results of an $E(3)$-equivariant Steerable CNN with implicit kernels on the QM9 dataset. Both models we reported here have approx. $2\cdot10^6$ parameters and only differ in length and width of the convolutional layers. For non-energy regression tasks ($\alpha$, $\Delta \varepsilon$, $\varepsilon_{HOMO}$, $\varepsilon_{LUMO}$, $\mu$ and $C_\nu$), we obtain results that are on par with the best-performing message-passing based approaches (we do not compare against transformers). We also indicate that Steerable CNNs with implicit kernels significantly outperform steerable linear convolutions (TFN, LieConv, L1Net) on most tasks. This is consistent with the observation of Brandstetter \etal \cite{Brandstetter2022GeometricAP}, who pointed out that non-linear convolutions generally perform better than linear ones. For the remaining energy variables (G, H, U, $U_0$, ZPVE) and $R^2$, our model significantly falls behind the task-specific benchmark approaches. We theorize that it can be attributed to two factors. First, steerable convolutions generally do not perform well on these tasks compared to problem-tailored frameworks (PaiNN \cite{Schtt2021EquivariantMP}, DimeNet++ \cite{Klicpera2020FastAU}, SphereNet \cite{Liu2021SphericalMP}), also in the non-linear case (e.g. SEGNN \cite{Brandstetter2022GeometricAP}). Second, we hypothesize that molecular connectivity does not produce a sufficient number of atom-atom interactions, which is crucial for the performance of a message-passing-based model \cite{Brandstetter2022GeometricAP}. However, as the goal of the section was to demonstrate the flexibility of implicit kernels that can be conditioned on features of graph edges, we leave developing more involved architectures (e.g. with induced connectivity) for further work. 

\section{Conclusion}
We propose a novel approach for implementing convolutional kernels of Steerable CNNs, allowing for the use of smaller groups and easy integration of additional features under the same framework with minor changes. To avoid analytically solving the group $G$-specific equivariance constraint, we use a $G$-equivariant MLP to parameterize a $G$-steerable kernel basis. We theoretically prove that MLP equivariance is sufficient for building an equivariant steerable convolutional layer. Our implicit representation outperforms a previous general method, offering a way to implement equivariance to various groups for which a custom kernel basis has not been developed. The N-body experiment suggests that this method will be particularly applicable in scenarios where rotational symmetry is disturbed, such as in material science or computational chemistry. The critical force term, which violated rotational symmetry, was effectively captured by our model, while the state-of-the-art model struggled with it.
Additionally, our flexible neural representation enables the introduction of arbitrary features into a convolutional kernel, enhancing the expressivity of Steerable CNNs. We validate this advantage by applying Steerable CNNs to point cloud and molecular graph data and achieving competitive performance with state-of-the-art approaches. In conclusion, we present a simple yet efficient solution for constructing a general kernel basis equivariant to an arbitrary compact group.

\section*{Limitations}
Steerable CNNs generally have high computational complexity and higher memory requirements compared to traditional CNNs. Implicit neural kernels do not help to mitigate the issue, yet they provide additional control over the kernel complexity. We found that when the kernels are parameterized by a single linear layer, the run time slightly decreases compared to the analytical solution\footnote{ as implemented in \href{https://quva-lab.github.io/escnn/}{\texttt{escnn}} \cite{Cesa2022APT}}, while a relative performance gain remains. We suggest that more efficient ways to implement $G$-MLPs would significantly contribute to the acceleration of the method and leave it for further research. From the implementation point of view, the most troublesome was achieving initialization similar to the non-implicit convolutions. We expect the problem to manifest even stronger when dealing with dense convolutions. This, combined with the discretization error of matrix filters, might negatively affect the performance. We are, however, convinced that it is a matter of time before a robust way of initializing $G$-equivariant implicit kernels will be obtained.

\begin{ack}
All the experiments were performed using the Hemera compute cluster of Helmholtz-Zentrum Dresden-Rossendorf and the IvI cluster of the University of Amsterdam. This research results from a collaboration initiated at the London Geometry and Machine Learning Summer School 2022 (LOGML). The authors thank Anna Mészáros, Chen Cai and Ahmad Hammoudeh for their help at the initial stage of the project. They also thank Rob Hesselink for his assistance with visualizations.
\end{ack}

\newpage
\bibliographystyle{ieee_fullname}
\bibliography{main}

\newpage
\appendix
\clearpage
\setcounter{table}{0}
\renewcommand{\thetable}{A\arabic{table}}

\section{Theoretical details}
This section provides an additional mathematical background that might be useful for understanding Steerable CNNs (Section~\ref{sec:app.def}). Besides, we write down the proof of the cornerstone Lemma \ref{eq:lemma}, which allows the application of implicit kernels (Section~\ref{sec:app.proof}). A more comprehensive introduction to representation theory and Steerable CNNs can be found in \cite{weiler2023EquivariantAndCoordinateIndependentCNNs}. Finally, we highlight the difference between our method and the one described in \cite{Cesa2022APT} in terms of prerequisites in Section~\ref{sec:app.comparison_analytical}.
\subsection{Additional details and definitions on group theory}
\label{sec:app.def}
\begin{definition}[Group]
A group is an algebraic structure that consists of a set $G$ and a binary operator $\circ: G \times G \rightarrow G$ called the group product (denoted by juxtaposition for brevity $g \circ h = gh$) that satisfies the following axioms:
\begin{multicols}{2}
\begin{itemize}
  \item $\forall g,h \in G: gh \in G$;
  \item $\exists! \; e \in G: eg=ge=g \; \forall g \in G$;
  \item $\forall g \in G \; \exists! \; g^{-1} \in G: gg^{-1}=g^{-1}g=e$;
  \item $(gh)k = g(hk) \; \forall g, h, k \in G$.
\end{itemize}
\end{multicols}
\end{definition}

\begin{example}[The Euclidean group $E(3)$]
The 3D Euclidean group $E(3)$ comprises three-dimensional translations, rotations, and reflections. These transformations are defined by a translation vector $x\in \mathbb{R}^3$ and an orthogonal transformation matrix $R \in \text{O}(3)$. The group product and inverse are defined as follows:
\begin{multicols}{2}
\begin{itemize}
    \item $g \cdot g' := (R x' + x, R R') \ $;
    \item $g^{-1} := (R^{-1} x, R^{-1}) \ $,
\end{itemize}
\end{multicols}
where $g=(x,R)$ and $g'= (x',R')$ are elements of $E(3)$. These definitions satisfy the four group axioms, establishing E(3) as a group. The action of an element $g \in E(3)$ on a position vector $y$ is given by:
\begin{equation*}
g \cdot y:= R y + x,
\end{equation*}
where $g = (x,R)$ is an element of $E(3)$ and $y \in \mathbb{R}^3$.
\end{example}
\vspace{5pt}
\begin{definition}[Semi-direct product]
Let $N$ and $H$ be two groups, each with their own group product, which we denote with the same symbol $\cdot \;$, and let $H$ act on $N$ by the action $\odot$. Then a (outer) semi-direct product $G = N \rtimes H$, called the semi-direct product of $H$ acting on $N$ , is a group whose set of elements is
the Cartesian product $N \times H$, and which has group product and inverse:
\begin{equation*}
(n,h) \cdot (\hat n, \hat h) = (n \cdot (h \odot \hat n), h \cdot \hat h)
\end{equation*}
\begin{equation*}
(n,h)^{-1} = (h^{-1} \odot n^{-1}, h^{-1})
\end{equation*}
\end{definition}
\noindent for all $n,\hat n \in N$ and $h,\hat h \in H$.

\subsection{Proof of the Lemma \ref{eq:lemma}}
\label{sec:app.proof}
\begin{lemma*}
If a kernel $k$ is parameterized by a $G$-equivariant MLP $\phi$ with input representation $\rho_{st}$ and output representation $\rho_\otimes := \rho_{in} \otimes \rho_{out}$ , i.e. $vec(k)(x) := \phi(x)$, then the kernel satisfies the equivariance constraint in Equation~\ref{eq:eq.constraint} for a compact group $G$.
\end{lemma*}
\begin{proof}
\label{eq:proof}
By construction, the equivariant MLP satisfies
\begin{equation}
    \phi(\rho_{st}(g) x) = \left(\rho_{in}(g) \otimes \rho_{out}(g)\right) \phi(x) \quad \forall g \in G, x\in \mathbb{R}^n
\end{equation}
We further use the substitution and $\phi \mapsto vec(k(\cdot))$ and find:
\begin{equation}
    vec(k(g.x)) = (\rho_{in}(g) \otimes \rho_{out}(g)) vec (k(x))   
\end{equation}
Now, we make use of the following identity describing the vectorization of a product of multiple matrices, which is the property of the Kronecker product:
\begin{equation}
\label{eq:vec.id}
    vec(ABC) = (C^T \otimes A) \: vec(B)
\end{equation}
Hence, identity \ref{eq:vec.id} allows us to re-write the previous equation as follows:
\begin{equation}
    k(g.x) = \rho_{out}(g) k(x) \rho_{in}(g)^{T}
\end{equation}    
Since we assume $G$ to be compact, its representations can always be transformed to an orthogonal form for which $\rho(g)^T = \rho(g)^{-1}$ via a change of basis. Hence, we find the equivariance constraint defined in equation \ref{eq:eq.constraint}.
\end{proof}

\subsection{Additional details on the comparison with \cite{Cesa2022APT}}
\label{sec:app.comparison_analytical}

To summarize the difference between our method and \cite{Cesa2022APT}, we provide Table~\ref{tab:req7}. There, we highlight the key ingredients required for the implementation of Steerable kernels for a group $G$ in both methods and estimate the "hardness" of obtaining each ingredient. As can be seen, the method described in \cite{Cesa2022APT} required a $G$-steerable basis for $L^2(\mathbb{R}^n)$, while implicit kernels do not. Instead, one only has to provide $G$-equivariant non-linearities, which we assume are available since those are the same non-linearities that will be used in the main model.

\begin{table}[h]
\caption{\label{tab:req7}Key ingredients required to build $G$-steerable kernels with baseline [7] (centre left) vs implicit kernels (centre right). The left column highlights the general prerequisites of Steerable CNNs, and the right column indicates the relative complexity of each ingredient.}
\vspace{5pt}
\centering
\begin{tabular}{ll}
\hline& \\[\dimexpr-\normalbaselineskip+2pt]
Requirement & Hardness \\
\hline& \\[\dimexpr-\normalbaselineskip+2pt]
\multicolumn{2}{c}{\textbf{Design of Steerable CNN architecture} \dag}   \\
\hline& \\[\dimexpr-\normalbaselineskip+2pt]
irreps $\hat{G}$ of $G$ & assumed \\
action of $G$ on $\mathbb{R}^n$ & assumed \\
$G$-equivariant non-linearities & assumed \\
\hline& \\[\dimexpr-\normalbaselineskip+2pt]
\multicolumn{2}{c}{\textbf{Solve constraint with \cite{Cesa2022APT}}} \\
\hline& \\[\dimexpr-\normalbaselineskip+2pt]
CG coefficients for $G$ & numerical\\
intertwiners $E_G(V_\psi)$ for $\psi \in \hat{G}$ & numerical or analytical \\
irreps-decomposition of $\rho_{in}$ and $\rho_{out}$ & numerical \\
$G$-steerable basis for $L^2(\mathbb{R}^n)$ & \textcolor{red!80!black}{handcrafted ad-hoc for each $G$}\\
\hline& \\[\dimexpr-\normalbaselineskip+2pt]
\multicolumn{2}{c}{\textbf{Implicit Kernel (Ours)}} \\
\hline& \\[\dimexpr-\normalbaselineskip+2pt]
CG coefficients for $G$ & numerical \\
intertwiners $E_G(V_\psi)$ for $\psi \in \hat{G}$ & numerical or analytical \\
irreps-decomposition of $\rho_{in}$ and $\rho_{out}$ & numerical \\
$G$-equivariant non-linearities & \textcolor{green!50!black}{available from \dag}\\
\hline
\end{tabular}
\end{table}

\section{Experimental details}
The section aims to provide additional details on model implementation for each particular experiment in Section~\ref{sec:experiments}. First, we describe how we preprocess the input of implicit kernels (relative position $x_i-x_j$, node features $z_i$, $z_j$, and edge features $z_{ij}$), which holds for every model (Section~\ref{sec:app.prep}). Then, we report the architectural details for models used in every experiment (Section~\ref{sec:app.exps}). 

\subsection{Preprocessing kernel's input}
\label{sec:app.prep}
An implicit kernel receives as input the relative position $x_i-x_j$, node features $z_i$ and $z_j$, and edge features $z_{ij}$. The first argument is a set of $3$-dimensional points transforming according to the standard representation $\rho_{st}$. We first compute its \emph{homogeneous polynomial representation} in $\mathbb{R}^3$ up to order $L$ and then batch-normalize it separately for each irrep before passing it to the implicit kernel. The harmonic polynomial $Y_l(x)$ of order $l$ evaluated on a point $x \in \mathbb{R}^3$ is a $2l+1$ dimensional vector transforming according to the Wigner-D matrix of frequency $l$ (and parity $l \mod 2$, when interpreted as an irrep of $O(3)$).
The vector $Y_l(x)$ is computed by projecting $x^{\otimes l} \in \mathbb{R}^{3^l}$ on its only $2l+1$ dimensional subspace transforming under the frequency $l$ Wigner-D matrix.\footnote{
    If $D_l$ is the frequency $l$ Wigner-D matrix, $x$ transforms under $D_l$, which is isomorphic to the standard representation of $SO(3)$.
    Then, $x^{\otimes l}$ transforms under $D_1^{\otimes l}$, i.e. the tensor product of $l$ copies of $D_1$.
    The tensor product of two Wigner-D matrices is well known to decompose as $D_l \otimes D_j \cong \bigoplus_{i=|l-j|}^{l+j} D_i$. By applying this rule recursively, one can show that $D_1^{\otimes l}$ contains precisely one copy of $D_l$.
    We define $Y_l(x)$ as the linear projection of $x^{\otimes l}$ to this subspace.
    Note also that, since this is a linear projection, the definition of $Y_l$ satisfies the defining property of \emph{homogeneous polynomials} $Y_l(\lambda x) = \lambda^l Y_l(x)$.
}
We keep $L = 3$ as we found the choice to be favourable for overall performance on validation data for ModelNet-40 and QM9 experiments. In the N-body experiment, $L=1$ for a fair comparison with the baseline. The remaining arguments form a vector with a pre-defined representation, which we concatenate to the harmonic representation. We extensively compared different preprocessing techniques, and the batch normalization of polynomial representation improved the performance of implicit kernels the most. We attribute it to higher numerical stability as we discovered that the standard deviation of MLP's output is the lowest in the case. The output of the implicit representation is a vector which we multiply with a Gaussian radial shell $\phi(x) = exp (-0.5 \cdot||x||_2^2 / \sigma^2)$ where $\sigma$ is a learnable parameter. This is coherent with the kernel basis typically used in literature \cite{Weiler2018-STEERABLE} - spherical harmonics modulated by a Gaussian radial shell.

\subsection{Model implementation}
\label{sec:app.exps}
\paragraph{N-body}

\begin{table}
\caption{\label{tab:app_nbody}Mean square error in the N-body system experiment vs stiffness of the strings from particles to the $XY$-plane. Stiffness practically indicates the degree of breaking the $SO(3)$ symmetry.}
\centering
\resizebox{\linewidth}{!}{%
\begin{tabular}{cccccccccc}
\hline
Stiffness & 0 & 1 & 5 & 10& 25& 50& 100    & 200    & 1000   \\ \hline
MPNN& 0.0022 & 0.0031 & 0.0068 & 0.0087 & 0.0030 & 0.0162 & 0.0560 & 0.0978 & 0.1065 \\
O(3)-SEGNN & 0.0009 & 0.0092 & 0.0117 & 0.0183 & 0.0291 & 0.0151 & 0.0229 & 0.0938 & 0.1313 \\ 
SO(2)-CNN-IK & 0.0010 & 0.0010 & 0.0009 & 0.0008 & 0.0008 & 0.0008 & 0.0014 & 0.0043 & 0.0162 \\ \hline
\end{tabular}
}
\end{table}

We used the reported configuration of the SEGNN model according to the official repository \cite{Brandstetter2022GeometricAP}, which has around $10^{4}$ parameters. We only modified the input of the model such that it takes the equilibrium length of the attached XY spring instead of the product of charges as in the original formulation, which was a trivial representation as well. As a result, the input representation consisted of 2 standard representations (position and velocity) and a trivial representation (spring's equilibrium length). The training was performed precisely according to the official configuration. For the non-equivariant baseline, we substituted every equivariant MLP in SEGNN with its non-equivariant counterpart and adjusted the number of parameters to match the one of the original.

To form a dataset, 3000 trajectories were generated with random initial velocities and equilibrium lengths of XY-plane strings for training and 128 for validation and testing. $G$-equivariant MLPs had 3 layers with 16 hidden fields. We used an embedding linear layer followed by 4 steerable convolutions and a $G$-equivariant MLP with 2 hidden layers applied to each node separately. The hidden representation was kept the same across every part of the model and had 16 steerable vector fields transforming under the spherical quotient representation band-limited to maximum frequency $1$. The total number of parameters was approximately equal to the one of the baseline. The model returned a coordinate vector transforming under the standard representation for each particle as output. We trained each model using a batch size of 128 for 200 epochs until reaching convergence. As in the case of SEGNN, we minimized the MSE loss. We used AdamW optimizer with an initial learning rate of $10^{-2}$. The learning rate was reduced by $0.5$ every 25 epochs. The training time, on average, was 5 min.

\paragraph{ModelNet-40}
\begin{wraptable}{R}{0.65\textwidth}
\vspace{-15pt}
\centering
\caption{Number of channels in each convolutional layer of a steerable CNN (see section \ref{sec:mn}).}
\vspace{15pt}
\label{tab:exp2_setup}
\begin{tabular}{llcc}
\hline
G  & kernel & channels & \#par, $\cdot 10^3$ \\ \hline& \\[\dimexpr-\normalbaselineskip+2pt]
$M$         & \begin{tabular}[c]{@{}l@{}}Implicit\\ Standard\end{tabular} & \begin{tabular}[c]{@{}c@{}}20 20 20 20 20 128\\ 20 20 20 20 20 128\end{tabular} & \begin{tabular}[c]{@{}c@{}}122\\ 144\end{tabular} \\
$Inv$       & \begin{tabular}[c]{@{}l@{}}Implicit\\ Standard\end{tabular} & \begin{tabular}[c]{@{}c@{}}20 20 20 20 20 128\\ 20 20 20 20 20 128\end{tabular} & \begin{tabular}[c]{@{}c@{}}122\\ 144\end{tabular} \\
$SO(2) \rtimes Inv$ & \begin{tabular}[c]{@{}l@{}}Implicit\\ Standard\end{tabular} & \begin{tabular}[c]{@{}c@{}}11 11 12 12 12 128\\ 20 25 25 30 30 128\end{tabular} & \begin{tabular}[c]{@{}c@{}}588\\ 557\end{tabular} \\
$SO(2)$     & \begin{tabular}[c]{@{}l@{}}Implicit\\ Standard\end{tabular} & \begin{tabular}[c]{@{}c@{}}15 15 20 30 30 128\\ 30 30 40 40 40 128\end{tabular} & \begin{tabular}[c]{@{}c@{}}565\\ 561\end{tabular} \\
$SO(2) \rtimes F$   & \begin{tabular}[c]{@{}l@{}}Implicit\\ Standard\end{tabular} & \begin{tabular}[c]{@{}c@{}}12 12 12 12 12 128\\ 20 25 25 30 30 128\end{tabular} & \begin{tabular}[c]{@{}c@{}}580\\ 557\end{tabular} \\
$SO(2) \rtimes M$   & \begin{tabular}[c]{@{}l@{}}Implicit\\ Standard\end{tabular} & \begin{tabular}[c]{@{}c@{}}15 20 20 20 20 128\\ 30 40 40 40 40 128\end{tabular} & \begin{tabular}[c]{@{}c@{}}592\\ 592\end{tabular} \\
$SO(3)$     & \begin{tabular}[c]{@{}l@{}}Implicit\\ Standard\end{tabular} & \begin{tabular}[c]{@{}c@{}}10 10 10 10 20 128\\ 30 30 30 30 30 128\end{tabular} & \begin{tabular}[c]{@{}c@{}}160\\ 154\end{tabular} \\
$O(3)$      & \begin{tabular}[c]{@{}l@{}}Implicit\\ Standard\end{tabular} & \begin{tabular}[c]{@{}c@{}}10 10 10 10 20 128\\ 30 30 30 30 30 128\end{tabular} & \begin{tabular}[c]{@{}c@{}}140\\ 128\end{tabular}
\end{tabular}
\vspace{-15pt}
\end{wraptable}

In the generalizability experiments described in Section \ref{sec:mn}, we maintained the configuration of the $G$-equivariant MLP as follows: two linear layers with 8 hidden fields and spherical quotient ELU with a maximum frequency of 2 in between. Each model consisted of an embedding linear layer, 6 steerable convolutions followed by spherical quotient ELU and batch normalization, and an MLP. The initial layer took the normals of each point in the point cloud with the standard representation as input. We utilized steerable vector spaces up to order 2 in each convolutional layer. To ensure comparability, we only varied the number of channels in each layer, aiming for a similar overall number of parameters among the models. The number of channels for each layer and the total number of parameters for each model are provided in Table \ref{tab:exp2_setup}. The last layer generated a 128-dimensional vector comprising scalar features for each node. Global max pooling was applied to obtain a 128-dimensional embedding of the point cloud. We further employed a 2-layer MLP (128 $\xrightarrow{ELU}$ 128 $\xrightarrow{ELU}$ 40) to calculate the class probability. In each convolutional layer, the point cloud was downsampled, resulting in the following sequence of input points: $1024\rightarrow256\rightarrow64\rightarrow64\rightarrow64\rightarrow16\rightarrow16$. 

For the experiment results presented in Table \ref{tab:mn}, we scaled up the $SO(2)$-equivariant model by increasing the number of layers in implicit kernels to 3. We also used 6 convolutional layers with 20 channels each and employed the following downsampling: $1024\rightarrow256\rightarrow128\rightarrow128\rightarrow128\rightarrow64\rightarrow64\rightarrow64$. Additionally, skip connections were added between the layers, maintaining the same number of input points.

We trained each model using batch size 32 for 200 epochs, which we found to be sufficient for convergence. We minimized the cross entropy loss with label smoothing  \cite{Wang2019DynamicGC}. The position of each point in a point cloud is normalized to the interval $\left[-1,1\right]^3$ during the preprocessing. We used AdamW optimizer \cite{Loshchilov2019DecoupledWD} with an initial learning rate $10^{-3}$. We also decayed the learning rate by $0.5$ after every 25 epochs. The training time on average was 3h on an NVIDIA Tesla V100 GPU and varied across different $G$. For the experiment reported in Table~\ref{tab:mn}, the training time was around 30 hours on the same single GPU.

\paragraph{QM9}
During the preprocessing, we normalized the target variable by subtracting the mean and dividing it by the standard deviation computed on the training dataset. Each model has the following structure: embedding layer $\rightarrow$ steerable convolutional layers $\rightarrow$ global mean pooling $\rightarrow$ 2-layer MLP. The embedding layer consists of 3 parts: linear layer $\rightarrow$ learnable tensor product $\rightarrow$ spherical quotient ELU. First, it takes one-hot encoding of an atom type and applies a linear transformation. The learnable tensor product computes the tensor product of each field with itself to generate an intermediate feature map. Then, a learnable linear projection is applied to the feature map, which yields a map from trivial representations to spherical representations up to order $L=2$. In the next step, a sequence of convolutional layers with skip-connection and quotient ELU non-linearities is applied. We use steerable feature vectors with a maximum order of 2 in each convolutional layer. The final classification MLP is defined as follows: 128 $\xrightarrow{ELU}$ 128 $\xrightarrow{ELU}$ 1.

For the flexibility experiment, we employ steerable CNNs with 11 convolutional layers with residual connections and 24 channels. For the final performance indicated in Table \ref{tab:exp4}, we scale the model up and increase its overall width and length. Implicit kernels are parameterized with $O(3)$-equivariant MLP with 3 linear layers with 16 fields and spherical quotient ELU in between. Concerning the number of parameters, the model with standard steerable kernels had $140$k parameters, 1-layer $G$-MLP - $356$k parameters, 2-layer $G$-MLP - $1.1$M parameters, 3-layer $G$-MLP - $1.2$M parameters.

We optimized the number of layers and channels for the $\varepsilon_{HOMO}$ regression task and the number of training epochs for the $G$ regression task. We trained each model using batch size 128 for either 125 epochs (the flexibility experiment) or 250 epochs (the final experiment). Each model is optimized with AdamW optimizer with an initial learning rate of $5\cdot10^{-4}$. We use learning rate decay by 0.5 every 25 epochs. It takes around 20 minutes per epoch on an NVIDIA Tesla V100 GPU.

\subsection{Depth-width trade-off.}
Romero \etal \cite{Romero2022CKConvCK} indicated that implicit representation of convolutional kernels allows one to build a shallower model compared to standard CNNs. We, however, did not obtain a similar pattern. Even though the width of standard steerable convolutional kernels must be pre-specified, keeping it sufficiently large yields the same scaling pattern as for implicit ones. We note that implicit kernels can adapt their width and thus the field of view, which is not the case for standard steerable kernels. We hypothesize that the result might change on different datasets where long-range dependencies play a more important role, e.g. in sequential data, as shown in \cite{Romero2022CKConvCK}.

\subsection{Complexity and training time}
\begin{figure}
\centering
\includegraphics[width=0.8\textwidth]{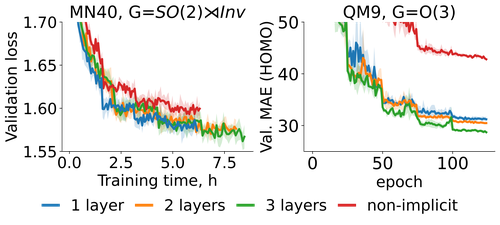}
\caption{\footnotesize Learning curve for different steerable kernels.}
\label{fig:learning_curve}
\end{figure} 
Overall, implicit kernels offer a flexible alternative to standard steerable kernels, potentially at an increased computational cost compared to optimized implementations of handcrafted steerable bases.
However, the additional cost depends only on the MLP's complexity and can be controlled during the design. We demonstrate the effect of the MLP's depth on performance and training time in Fig.~\ref{fig:learning_curve}. In practice, we find that two layers provide the best trade-off, but increased complexity might be beneficial when including additional attributes (Fig.~\ref{fig:learning_curve}, right). We also demonstrate that training and inference time are approximately equal for standard steerable kernels \cite{Weiler2018-STEERABLE} and implicit kernels parameterized by a single layer. Finally, we only experienced training challenges (e.g. instabilities) when using MLPs with $\#\text{layers} > 3$. To overcome these issues, we used batch normalization of spherical harmonics, i.e. the input of MLPs, which proved to be effective. 

Importantly, implicit kernels don't have substantially more hyperparameters compared to the method suggested in \cite{Cesa2022APT}. One only has to tune the parameters of a $G$-MLP, which are arguably easier to interpret than the parameters of handcrafted bases that are typically group-specific.

\end{document}